%% file: main.tex
\documentclass[11pt,a4paper]{article}
\usepackage[hyperref]{acl2021}
\usepackage{times}
\usepackage{latexsym}

\usepackage{microtype}
\usepackage{amsthm}
\usepackage{graphicx} % DO NOT CHANGE THIS
\usepackage{caption} % DO NOT CHANGE THIS AND DO NOT ADD ANY OPTIONS TO IT
\usepackage{subcaption}
\usepackage{scalerel}
\newtheorem{proposition}{Proposition}

\input{settings.tex}
\input{math_commands.tex}

\aclfinalcopy % Uncomment this line for the final submission
\def\aclpaperid{3590} %  Enter the acl Paper ID here

\newcommand\ele{\textsc{Electra}}
\newcommand\electric{\textsc{Electric}}
\newcommand{\subG}{\textnormal{\textsc{g}}}
\newcommand{\subS}{\textnormal{\textsc{s}}}

\title{Learning to Sample Replacements for \textsc{Electra} Pre-Training}

\author{Yaru Hao$^{\dag\ddag}$\thanks{\ \  Contribution during internship at Microsoft Research.},~~Li Dong$^\ddag$,~~Hangbo Bao$^\ddag$,~~Ke Xu$^{\dag}$,~~Furu Wei$^\ddag$\\
$^\dag$Beihang University \\
$^\ddag$Microsoft Research \\
\texttt{\{haoyaru@,kexu@nlsde.\}buaa.edu.cn} \\
\texttt{\{lidong1,t-habao,fuwei\}@microsoft.com} \\}

\date{}

\begin{document}
\maketitle
\begin{abstract}
\ele{}~\cite{electra} pretrains a discriminator to detect replaced tokens, where the replacements are sampled from a generator trained with masked language modeling. Despite the compelling performance, \ele{} suffers from the following two issues. First, there is no direct feedback loop from discriminator to generator, which renders replacement sampling inefficient. Second, the generator's prediction tends to be over-confident along with training, making replacements biased to correct tokens. In this paper, we propose two methods to improve replacement sampling for \ele{} pre-training. Specifically, we augment sampling with a hardness prediction mechanism, so that the generator can encourage the discriminator to learn what it has not acquired. We also prove that the efficient sampling reduces the training variance of the discriminator. Moreover, we propose to use a focal loss for the generator in order to relieve oversampling correct tokens as replacements. Experimental results show that our method improves \ele{} pre-training on various downstream tasks.
Our code and pre-trained models will be released at \url{https://github.com/YRdddream/electra-hp}
\end{abstract}

\section{Introduction}

One of the most successful language model pre-training tasks is masked language modeling (MLM; \citealt{bert}).
First, we randomly mask some input tokens in a sentence. Then the encoder learns to recover the masked tokens given the corrupted input.
\ele{}~\cite{electra} argues that MLM only produces supervision signals at a small proportion of positions (usually $15\%$), and uses the replaced token detection task as an alternative.
Specifically, \ele{} contains a generator and a discriminator. The generator is a masked language model, which substitutes masks with the tokens sampled from its MLM predictions. The discriminator learns to distinguish which tokens have been replaced or kept the same.
Experimental results on downstream tasks show that \ele{} can largely improve sample efficiency.

Despite achieving compelling performance, it is usually difficult to balance the training pace between the generator and the discriminator.
Along with pre-training, the generator is expected to sample more hard replacements for the detection task in a curriculum manner, while the discriminator learns to identify the corrupted positions.
Although the two components are designed to compete with each other, there is no explicit feedback loop from the discriminator to the generator, rendering the learning games independent.
The absence of feedback results in sub-efficient learning, because many replaced tokens have been successfully trained while the generator does not know how to effectively sample replacements.
In addition, a well trained generator tends to achieve reasonably good MLM accuracy, where many sampled replacements are correct tokens.
In order to relieve the issue of oversampling correct tokens, \ele{} explored tweaking the mask probability larger, raising the sampling temperature, and using a manual rule to avoid sampling original tokens.

In this paper, we propose two methods, namely hardness prediction and sampling smoothing, to tackle the above issues.
First, the motivation of \textit{hardness prediction} is to sample the replacements that the discriminator struggles to predict correctly.
We elaborate on the benefit of a good replacement mechanism from the perspective of variance reduction.
Theoretical derivations indicate that the replacement sampling should be proportional to both the MLM probability (i.e., language frequency) and the corresponding discriminator loss (i.e., discrimination hardness).
Based on the above conclusion, we introduce a sampling head in the generator, which learns to sample by estimating the expected discriminator loss for each candidate replacement.
So the discriminator can give feedback to the generator, which helps the model to learn what it has not acquired.
Second, we propose a \textit{sampling smoothing} method for the issue of oversampling original tokens.
We adopt a focal loss~\cite{focalloss} for the generator's MLM task, rather than using cross-entropy loss.
The method adaptively downweights the well-predicted replacements for MLM, which avoids sampling too many correct tokens as replacements.

We conduct pre-training experiments on the WikiBooks corpus for both small-size and base-size models.
The proposed techniques are plugged into \ele{} for training from scratch.
Experimental results on various tasks show that our methods outperform \ele{} despite the simplicity.
Specifically, under the small-size setting, our model performance is $0.9$ higher than \ele{} on MNLI~\cite{mnli2017} and $4.2$ higher on SQuAD 2.0~\cite{squad}, respectively.
Under the base-size setting, our model performance is $0.26$ higher than \ele{} on MNLI and $0.52$ higher on SQuAD 2.0, respectively.

\section{Related Work}

State-of-the-art NLP models are mostly pretrained on a large unlabeled corpus with the self-supervised objectives~\cite{elmo,albert,t5}.
The most representative pretext task is masked language modeling (MLM), which is introduced to pretrain a bidirectional BERT~\cite{bert} encoder.
RoBERTa~\cite{roberta} apply several strategies to enhance the BERT performance, including training with more data and dynamic masking.
UniLM~\cite{unilm, unilm2} extend the mask prediction to generation tasks by adding the auto-regressive objectives.
XLNet~\cite{xlnet} propose the permuted language modeling to learn the dependencies among the masked tokens.
Besides, \ele{}~\cite{electra} propose a novel training objective called replaced token detection which is defined over all input tokens.
Moreover, \electric{}~\cite{electric} extends the idea of \ele{} by energy-based cloze models.

Some prior efforts demonstrate that sampling more hard examples is conducive to more effective training.
\citet{focalloss} propose the focal loss in order to focus on more hard examples.
Generative adversarial networks~\cite{gan} is trained to maximize the probability of the discriminator making a mistake, which is closely related to \ele{}'s training framework.
In this work, we aim at guiding the generator of \ele{} to sample the replacements that are hard for the discriminator to predict correctly, therefore the pre-training process of the discriminator can be more efficient.

\section{Background: \ele{}}
\label{sec:electra}

\begin{figure*}[t]
\centering
\includegraphics[width=0.98\linewidth]{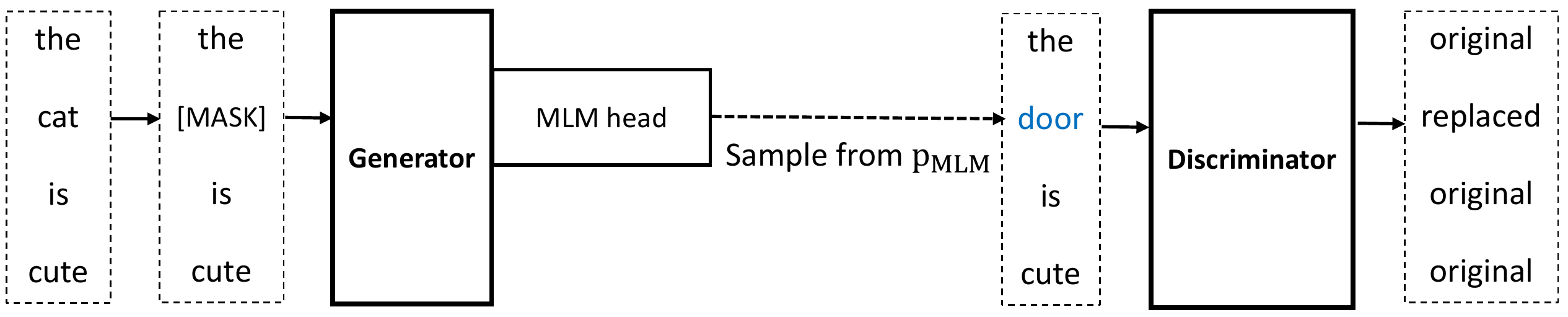}
\caption{An overview of \ele{}. 
The MLM head of the generator learns to perform MLM and samples replacements at each masked position from the MLM distribution. 
For the corrupted sequence, the discriminator learns to distinguish which tokens have been replaced or kept the same.}
\label{fig:ele_model}
\end{figure*}

\begin{figure*}[t]
\centering
\includegraphics[width=0.98\linewidth]{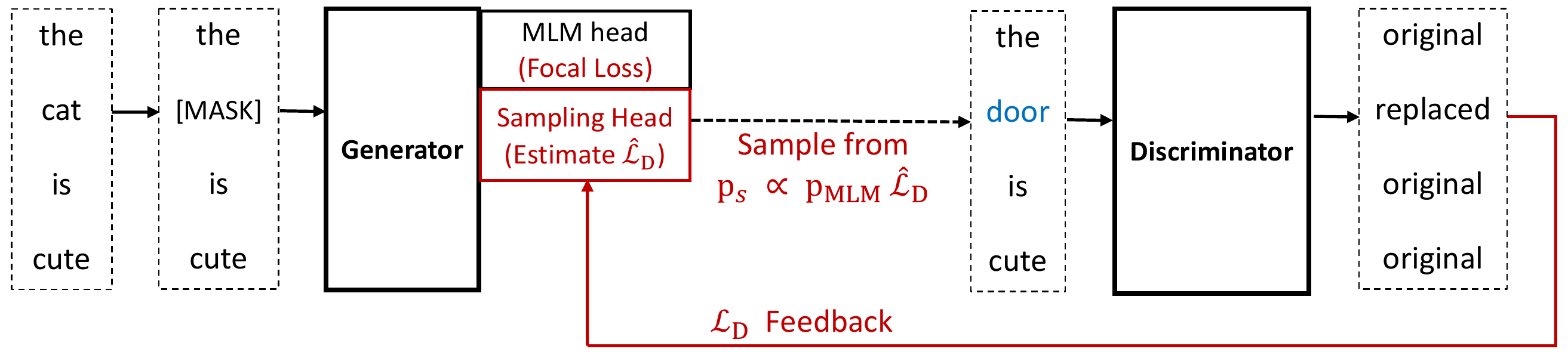}
\caption{An overview of our model.
The generator has two prediction heads. 
The MLM head learns to perform MLM through the focal loss instead of the cross entropy loss.
The sampling head is trained to estimate the discriminator loss over the vocabulary. 
Our model samples the replacements from a new distribution, which is proportional to both the MLM probability and the corresponding discriminator loss.
The discriminator is trained to distinguish input tokens and the loss feedback is transferred to the generator for the sampling head to learn.}
\label{fig:our_model}
\end{figure*}

An overview of \ele{} is shown in Figure~\ref{fig:ele_model}.
The model consists of a generator \emph{G} and a discriminator \emph{D}.
The generator is trained by masked language modeling (MLM).
Formally, given an input sequence $\boldsymbol{x} = x_1 \cdots x_n$, we first randomly mask $k=\lceil 0.15n \rceil$ tokens at the positions $\boldsymbol{m} = m_1 \cdots m_k$ with \sptk{MASK}.
The perturbed sentence $\boldsymbol{c}$ is denoted as:
\begin{align*}
&m_i \sim \mathrm{uniform}\{1, n\} \quad \mbox{for} \; i=1 , \cdots , k  \\
&\boldsymbol{c} = \mathrm{replace}(\boldsymbol{x}, \boldsymbol{m}, \sptk{MASK})
\end{align*}
where the $\mathrm{replace}$ operation conducts masking at the positions $\boldsymbol{m}$.
The generator encodes $\boldsymbol{c}$ and performs MLM prediction.
At each masked position $i$, we sample replacements from MLM output distribution $p_{\subG}$:
\begin{align*}
&{x'}_{i} \sim p_{\subG}(x_{i}'|\boldsymbol{c}) \quad \mbox{for} \; i \in \boldsymbol{m}  \\
&{\boldsymbol{x}}^{\rm R} = \mathrm{replace}(\boldsymbol{x}, \boldsymbol{m}, {\boldsymbol{x'}}) 
\end{align*}
where masks are replaced with the sampled tokens.
Next, the discriminator encodes the corrupted sentence ${\boldsymbol{x}}^{\rm R}$.
A binary classification task learns to distinguish which tokens have been replaced or kept the same, which predicts the probability $D(x^{\rm R}_t, \boldsymbol{x}^{\rm R})$ to indicate how likely $x^{\rm R}_t$ comes from the true data distribution.

The overall pre-training objective is defined as:
\begin{equation*}
{
\begin{aligned}
&\underset{\theta_G, \theta_D}{\rm min}\!\sum_{\boldsymbol{x} \in \mathcal{X}}\!\underset{i \in \boldsymbol{m}}{\rm \mathbb{E}}[\mathcal{L}_G^{\theta_G}(x_i,\! \boldsymbol{c})]\!+\!\lambda\!\underset{t \in [1,n]}{\rm \mathbb{E}}[\mathcal{L}_D^{\theta_D}(x^{\rm R}_t\!,\! \boldsymbol{x}^{\rm R})] \\
&\mathcal{L}_G(x_i, \boldsymbol{c})\!=\!-\!\log p_{\subG}(x_i|\boldsymbol{c}) \\
&\mathcal{L}_D(x^{\rm R}_t, \boldsymbol{x}^{\rm R})\!=\! 
\begin{cases}
    \!-\!\log D(x^{\rm R}_t, \boldsymbol{x}^{\rm R}) &x^{\rm R}_t\!=\!x_t \\
    \!-\!\log (1\!-\!D(x^{\rm R}_t, \boldsymbol{x}^{\rm R}))&x^{\rm R}_t\!\neq\!x_t \\
\end{cases}
\end{aligned}
}
\end{equation*}
where $\mathcal{X}$ represents text corpus, and $\lambda$ $=50$ suggested by \citet{electra} is a hyperparameter used to balance the training pace of generator and discriminator.
Once pre-training is finished, only the discriminator is fine-tuned on downstream tasks.

\section{Methods}
\label{sec:methods}

\subsection{Hardness Prediction}
\label{sec:hardness}

The key idea of hardness prediction is to let the generator receive the discriminator's feedback and sample more hard replacements.
Figure~\ref{fig:our_model} shows the overview of our method.
Besides the original MLM head in the generator, there is an additional sampling head used to sample replaced tokens.

Given a\footnote{For notation simplicity, we assume only one token is masked in each sentence.} replaced token $x'$ in the input sequence $c$, let ${\mathcal{L}}_D(x',c)$ denote the discriminator loss for the replacement.
Rather than directly sampling replacements from the MLM prediction $p_{\subG}$, we propose to sample from $p_{\subS}$:
\begin{align}
p_{\subS}(x'|c) &= \frac{p_{\subG}(x'|c)  {\mathcal{L}}_D(x',c)}{{\rm E}_{p_{\subG}(x^*|c)}[{\mathcal{L}}_D(x^*,c)]} \label{eq:opt_ps} \\
{\boldsymbol{x}}^{\rm R} &= \mathrm{replace}(\boldsymbol{x}, \boldsymbol{m}, {\boldsymbol{x'}}) \quad {x'} \sim p_{\subS}(x'|c) \nonumber
\end{align}
where the corrupted sentence ${\boldsymbol{x}}^{\rm R}$ is obtained by substituting the masked positions $\boldsymbol{m}$ with sampled replacements $\boldsymbol{x'}$.
The first term $p_{\subG}(x'|c)$ implies sampling from the data distribution.
The second term ${\mathcal{L}}_D(x',c)$ encourages the model to sample more replacements that the discriminator has not successfully learned.

Notice that \eqform{eq:opt_ps} uses the actual discriminator loss ${\mathcal{L}}_D(x',c)$, which can not be obtained without feeding ${\boldsymbol{x}}^{\rm R}$ into the discriminator.
As an alternative, we use the estimated loss value $\hat{\mathcal{L}}_D(x',c)$ to sample replaced tokens, which approximates the actual loss for the candidate replacement.
During pre-training, we use the actual loss as supervision, and simultaneously train the sampling head.
We describe the detailed implementations of loss estimation in Section~\ref{sec:implment:hardness}.

By considering detection hardness in replacement sampling and giving feedback from the discriminator to the generator, the components are no longer independently learned.
\ele{}~(\citealt{electra}; Appendix F) also attempts to achieve the same goal by adversarially training the generator.
However, it underperforms the maximum-likelihood training, because of the poor sample efficiency of reinforcement learning on discrete text data.
More importantly, their generator is trained to fool the discriminator, rather than guiding the discriminator by data distribution, which breaks the \ele{} training objective.
In contrast, we still retain the MLM head, and decouple it from replacement sampling. So we can take the advantage of the original training objective.

\subsubsection{Perspective of Variance Reduction}

We show that the proposed hardness prediction method is well supported from the perspective of variance reduction.

\begin{proposition}
Sampling replacements from $p_{\subS}(x'|c)$ can minimize the estimation variance of the discriminator loss.
\end{proposition}

\begin{proof}
At each masked position, the expectation of the discriminator loss we aim to estimate can be summarized as $\mathrm{Z} = {\rm E}_{p_{\subG}(x^*|c)}[\mathcal{L}_D(x^*, c)]$.
Under $p_{\subG}$, the estimation variance of the discriminator loss is:
\begin{align*}
& \qquad \Var_{p_{\subG}(x^*|c)}[\mathcal{L}_D(x^*, c)] \\
&= \sum_{x^* \in \text{vocab}} p_{\subG}(x^*|c) (\mathcal{L}_D(x^*,c) - \mathrm{Z})^2 \\
&= \sum_{x^* \in \text{vocab}} p_{\subG}(x^*|c) \mathcal{L}_D(x^*, c)^2 - \mathrm{Z}^2
\end{align*}

Similar to importance sampling, we can select an alternative distribution $p_{\subS}$ different from $p_{\subG}$, then the expectation $\mathrm{Z}$ is rewritten as:
\begin{align*}
&\qquad {\rm E}_{p_{\subG}(x^*|c)}[\mathcal{L}_D(x^*, c)] \\
&= \sum_{x^* \in \text{vocab}} p_{\subG}(x^*|c) \mathcal{L}_D(x^*,c) \\
&= \sum_{x^* \in \text{vocab}} p_{\subS}(x^*|c)  \frac{p_{\subG}(x^*|c)}{p_{\subS}(x^*|c)} \mathcal{L}_D(x^*,c) \\
&= {\rm E}_{p_{\subS}(x^*|c)}[\frac{p_{\subG}(x^*|c)}{p_{\subS}(x^*|c)} \mathcal{L}_D(x^*,c)]
\end{align*}
By making a multiplicative adjustment to $\mathcal{L}_D$, the estimation variance of $\mathrm{Z}$ under the new sampling distribution $p_{\subS}$ is converted to:
\begin{align*}
    &\qquad \Var_{p_{\subS}(x^*|c)}[\frac{p_{\subG}(x^*|c)}{p_{\subS}(x^*|c)} \mathcal{L}_D(x^*, c)] \\
    &= \sum_{x^* \in \text{vocab}} p_{\subS} (\frac{p_{\subG}(x^*|c)  \mathcal{L}_D(x^*,c)}{p_{\subS}})^2 - \mathrm{Z}^2 \\
    &= \sum_{x^* \in \text{vocab}} \frac{(p_{\subG}(x^*|c)  \mathcal{L}_D(x^*,c) - p_{\subS}(x^*|c) \mathrm{Z} )^2}{p_{\subS}(x^*|c)}
\end{align*}

Based on the above derivation, it is obvious that we obtain a zero-variance estimator when we choose $p_{\subS}(x^*|c)=p_{\subG}(x^*|c) \mathcal{L}_D(x^*,c) / \mathrm{Z}$ as \eqform{eq:opt_ps}.
This theoretically optimal form provides us insights into designing the above sampling scheme.
\end{proof}

\subsubsection{Two Implementations of Hardness Prediction}
\label{sec:implment:hardness}

We design two variants of the sampling head.
The first one is to explicitly estimate the discriminator loss (HP$_{\rm Loss}$).
The second method is to approximate the expected sampling distribution (HP$_{\rm Dist}$).

\paragraph{HP$_{\rm Loss}$} guides the generator to learn the probability predicted by the discriminator that the sampled token $x'$ is an original token.
In this case, the output layer of the sampling head is actually a $\mathrm{sigmoid}$ function same as the discriminator:
\begin{equation*}
\hat{D}(x',c) = \mathrm{sigmoid} (\boldsymbol{w}(x') \cdot \boldsymbol{h}_S(c))
\end{equation*}
where $\boldsymbol{h}_S(c)$ denotes the contextual representations projected by the sampling head, and $\boldsymbol{w}$ denotes the projection parameters.
Then the loss of the sampling head at the masked position is:
\begin{align*}
\mathcal{L}_S(x',c) = (\hat{D}(x',c) - D(x',c))^2 
\end{align*}

When sampling replacements over the vocabulary, the estimated discriminator probability $\hat{D}(x',c)$ can be easily rewritten to the estimated discriminator loss $\hat{\mathcal{L}}_D(x',c)$:
\begin{align*}
&\hat{\mathcal{L}}_D(x',c) = 
\begin{cases}
- \log \hat{D}(x',c) & x' = x \\
- \log (1 - \hat{D}(x',c)) & x' \neq x \\
\end{cases} 
\end{align*}
Multiplying the MLM probability factor $p_{\subG}$, we obtain the sampling distribution:
\begin{align*}
p_{\subS}(x'|c) &= \frac{p_{\subG}(x'|c) \hat{\mathcal{L}}_D(x',c)}{\sum_{x^* \in \text{vocab}} p_{\subG}(x^*|c) \hat{\mathcal{L}}_D(x^*,c)} \\
&= \frac{p_{\subG}(x'|c) \hat{\mathcal{L}}_D(x',c)}{{\rm E}_{p_{\subG}(x^*|c)}[\hat{\mathcal{L}}_D(x^*,c)]}
\end{align*}

\paragraph{HP$_{\rm Dist}$} aims to directly approximate the expected sampling distribution as in \eqform{eq:opt_ps}, instead of the discriminator loss.
In this case, the sampling head produces an output probability of the token $x'$ with a $\softmax$ layer:
\begin{equation}
\label{eq:hp:dist:softmax}
p_{\subS}(x'|c) = \frac{\exp (\boldsymbol{e}(x') \cdot \boldsymbol{h}_S(c))}{\sum_{x^* \in \text{vocab}} \exp (\boldsymbol{e}(x^*) \cdot \boldsymbol{h}_S(c))}
\end{equation}
where $\boldsymbol{e}$ represents the token embeddings.
For the sampled token $x'$, we define the loss of the sampling head as:
\begin{align*}
\mathcal{L}_S(x',c) = - \frac{p_{\subG}(x'|c)}{p_{\subS}(x'|c)} \mathcal{L}_D(x',c) \log p_{\subS}(x'|c)
\end{align*}
Then we show that minimizing the above loss $\mathcal{L}_S(x',c)$ pushes sampling distribution of \eqform{eq:hp:dist:softmax} to our goal.
Specifically, the loss expectation over the whole vocabulary is:
\begin{align*}
&{\rm E}_{p_{\subS}(x^*|c)}[\mathcal{L}_S(x^*,c)] = \\
&\!-\!\sum_{x^* \in \text{vocab}} p_{\subS}(x^*|c)\!\frac{p_{\subG}(x^*|c)}{p_{\subS}(x^*|c)} \mathcal{L}_D(x^*,c)\!\log p_{\subS}(x^*|c) \\
&= - \sum_{x^* \in \text{vocab}} p_{\subG}(x^*|c) \mathcal{L}_D(x^*,c) \log p_{\subS}(x^*|c)
\end{align*}
According to the Lagrange Multiplier method, the optimal solution $\tilde{p}_{\subS}$ of the loss function $\mathcal{L}_S(x',c)$ is consistent with \eqform{eq:opt_ps}:
\begin{align*}
\tilde{p}_{\subS}(x'|c) &= \frac{p_{\subG}(x'|c) \mathcal{L}_D(x',c)}{\sum_{x^* \in \text{vocab}} p_{\subG}(x^*|c) \mathcal{L}_D(x^*,c)} \\
&= \frac{p_{\subG}(x'|c) \mathcal{L}_D(x',c)}{{\rm E}_{p_{\subG}(x^*|c)}[\mathcal{L}_D(x^*,c)]}
\end{align*}

\subsection{Sampling Smoothing}
\label{sec:focal}

Along with the learning process, the masked language modeling tends to achieve relatively high accuracy.
As a consequence, the generator oversamples the correct tokens as replacements, which renders the discriminator learning inefficient.

In order to address the issue, we apply an alternative loss function called focal loss~\cite{focalloss} for MLM of the generator.
Compared with the vanilla cross-entropy loss, focal loss adds a modulating factor for the weighting purpose:
\begin{align*}
\mathcal{L}^{{\rm fc}}_G(x, c) = - (1 - p_{\subG}(x|c))^{\gamma} \log p_{\subG}(x|c) 
\end{align*}
where $\gamma \geq 0$ is a tunable hyperparameter.
Besides using a constant $\gamma$, we try the piecewise function $\gamma = \mathbbm{1}(p_{\subG}\!>\! 0.2) * 3 + \mathbbm{1}(p_{\subG} \leq 0.2) * 5$ in our experiments as suggested by~\citet{calibfocal}.

In other words, the focal loss is used to adaptively down-weight the well-classified easy examples and thus focusing on more difficult ones.
When applying the focal loss to the MLM head for the generator, we notice that if a token is easy for the generator to be predicted correctly, i.e., $p_{\subG}(x|c) \rightarrow 1$, the modulating factor is greatly decreased.
In contrast, if a token is hard to predict, the focal loss approximates to the original cross entropy loss.
Therefore, we propose to employ the focal loss in order to smooth the sampling distribution, which in turn relieves oversampling correct tokens as replacements.

\begin{table*}[t]
\centering
% \small
\scalebox{0.92}{
\begin{tabular}{l c c c c c c c c c}
\toprule
\multirow{2}{*}{\bf Model}   & \textbf{MNLI} & \textbf{QNLI} & \textbf{QQP} & \textbf{RTE} & \textbf{SST} & \textbf{MRPC} & \textbf{CoLA} & \textbf{STS}  \\
&      -m/-mm      &      Acc       &      Acc      &      Acc      &      Acc      &      Acc      &      MCC      &      PCC      \\ \midrule
~~\emph{Small-size} & \\
\ele{} (\emph{reimplementation}) & 80.3/80.6 & 89.0 & 89.2 & 62.2 & 89.0 & \textbf{87.3} & 59.6 & 86.3 
\\
\midrule
\ele{}+HP$_{\rm Loss}$+Focal & \textbf{81.2}/81.7 & 89.1 & \textbf{89.6} & \textbf{66.1} & \textbf{89.2} & 86.6 & 59.3 & 86.7  \\
\ele{}+HP$_{\rm Dist}$+Focal & 81.0/\textbf{81.8} & \textbf{89.3} & \textbf{89.6} & 62.5 & \textbf{89.2} & 86.8 & \textbf{59.7} & \textbf{86.8} \\
\midrule
~~\emph{Base-size} & \\
BERT~\cite{bert} & 84.5/- & 88.6 & 90.8 & 68.5 & 92.8 & 86.0 & 58.4 & 87.8 \\
RoBERTa~\cite{roberta} & 84.7/- & - & - & - & 92.7 & - & - & -  \\
XLNet~\cite{xlnet} & 85.8/- & - & - & - & \textbf{93.4} & - & - & - \\
\ele{}~\cite{electra} & 86.2/- & 92.4 & 90.9 & 76.3 & 92.4 & 87.9 & 65.8 & 89.1  \\
\electric{}~\cite{electric} & 85.7/- & 92.1 & 90.6 & 73.4 & 91.9 & 88.0 & 61.8 & 89.4 \\
\ele{} (\emph{reimplementation}) &  86.7/86.5 & 92.6 & 91.4 & 80.4 & 92.6 & 89.1 & 66.5 & \textbf{91.0} \\
\midrule
\ele{}+HP$_{\rm Loss}$+Focal & \textbf{87.0}/\textbf{86.9} & \textbf{92.7} & \textbf{91.7} & \textbf{81.3} & 92.6 & \textbf{90.7} & 66.7 & \textbf{91.0}  \\
\ele{}+HP$_{\rm Dist}$+Focal & 86.8/86.8 & 92.3 & 91.6 & 80.0 & 92.7 & 89.8 & \textbf{67.3} & 90.9 \\
\bottomrule
\end{tabular}}
\caption{Comparisons between our models and previous pretrained models on GLUE dev set. 
Reported results are medians over five random seeds.
}
\label{tbl:glue_res}
\end{table*}

\subsection{Pre-Training Objective}

Adopting the above two strategies, we jointly train the generator and the discriminator together as the original \ele{} model.
The word embeddings of them are still tied during the pre-training stage.
Formally, we minimize the combined loss over a large corpus $\mathcal{X}$:
\begin{equation*}
{
\begin{aligned}
&\underset{\theta_G, \theta_S, \theta_D}{\rm min} \sum_{\boldsymbol{x} \in \mathcal{X}} \Big( \underset{i \in \boldsymbol{m}}{\rm \mathbb{E}}[\mathcal{L}_G^{{\rm fc}, \theta_G}(x_i, \boldsymbol{c})] + \\ &\lambda_1 \underset{i \in \boldsymbol{m}}{\rm \mathbb{E}}[\mathcal{L}_S^{\theta_S}(x^{\rm R}_i, \boldsymbol{c})] + \lambda_2 \underset{t \in [1,n]}{\rm \mathbb{E}}[\mathcal{L}_D^{\theta_D}(x^{\rm R}_t, \boldsymbol{x}^{\rm R})] \Big) \\
\end{aligned}
}
\end{equation*}
where $\lambda_1, \lambda_2$ are two hyperparameters to adjust three parts of the loss.
We only search $\lambda_1$ value and keep $\lambda_2=50$ for the fair comparison with \ele{}.
After pre-training, we throw out the generator and only fine-tune the discriminator on the downstream tasks.

\section{Experiments}

\subsection{Setup}

We implement \ele{}+HP$_{\rm Loss}$/HP$_{\rm Dist}$+Focal on both the \emph{small-size} setting and the \emph{base-size} setting.
The two prediction heads share both the generator and the token embeddings, which avoids the unnecessary increase in model complexity.
We follow most settings as suggested in \ele{}~\citep{electra}.
In order to enhance the \ele{} baseline for a solid comparison, we add the relative position~\cite{t5}. Experimental results show that our methods can improve performance even on the enhanced \ele{} baseline.

We pretrain our models on the same text corpus as \ele{}, which is a combination of English Wikipedia and BooksCorpus~\cite{wikidata}.
We also adopt the N-gram masking strategy which is beneficial for MLM tasks.
The models are trained for 1M steps for small-size models and 765k steps for base-size models, so that the computation consumption can be similar to baseline models~\cite{electra}.
The base-size models are pretrained with $16$ V100 GPUs less than five days.
The small-size models are pretrained with $8$ V100 GPUs less than three days.
We use the Adam~\cite{adam} optimizer ($\beta_1=0.9, \beta_2=0.999$) with learning rate of 1e-4.
The value of $\lambda_2$ in the training objective is kept fixed at $50$ for a fair comparison with \ele{}.
For HP$_{\rm Loss}$, we search $\lambda_1$ in $\{5, 10, 20\}$, the best one is $5$.
For HP$_{\rm Dist}$, we keep $\lambda_1=1$.
We search the focal loss weight $\gamma$ in $\{1, 4\}$ on both the base-size and small-size model, the best configuration is $\gamma=1$.
The detailed pre-training configurations are provided in the supplemental materials.

\subsection{Results on GLUE Benchmark}

The General Language Understanding Evaluation (GLUE) benchmark~\cite{glue} is a collection of diverse natural language understanding (NLU) tasks, including inference tasks (MNLI, QNLI, RTE; \citealt{rte1,rte2,rte3,rte5,mnli2017, squad1}), similarity and paraphrase tasks (MRPC, QQP, STS-B; \citealt{mrpc2005, sts-b2017}), and single-sentence tasks (CoLA, SST-2; \citealt{cola2018, sst2013}).
The detailed descriptions of GLUE datasets are provided in the supplementary materials.
The evaluation metrics are Spearman correlation for STS-B, Matthews correlation for CoLA, and accuracy for the other GLUE tasks.

For small-size settings, we use the hyperparameter configuration as suggested in~\cite{electra}.
For base-size settings, we consider a limited hyperparameter searching for each task, with learning rates $\in \{$5e-5, 1e-4, 1.5e-4$\}$ and training epochs $\in \{3, 4, 5\}$.
The remaining hyperparameters are the same as \ele{}.
We report the median performance on the dev set over five different random seeds for each task.
All the results come from the single-task fine-tuning.
For more detailed fine-tuning configurations, please refer to the supplementary materials.

Results are shown in Table~\ref{tbl:glue_res}.
With the same configuration and pre-training data, for both the small-size and the base-size, our methods outperform the strong reimplemented \ele{} baseline by 0.6 and 0.4 on average respectively.
For the most widely reported task MNLI, our models achieve 87.0/86.9 points on the matched/mismatched set, which obtains 0.3/0.4 absolute improvements. The performance gains on the small-size models are more obvious than the base-size models, we speculate that is due to the learning of the small-size generator is more insufficient and suffers from the above issues more significantly.
The results demonstrate that our proposed methods can improve the pre-training of \ele{}.
In other words, sampling more hard replacements is more efficient than the original masked language modeling.

\subsection{Results on SQuAD 2.0}

The Stanford Question Answering Dataset (SQuAD; \citealt{squad}) is a reading comprehension dataset, each example consists of a context and a question-answer pair.
Given a context and a question, the task is to answer the question by extracting the relevant span from the context.
We only use the version 2.0 for evaluation, where some questions are not answerable.
We report the results of both the Exact-Match (EM) and F1 score.
When fine-tuning on SQuAD, we add the question-answering module from XLNet on the top of the discriminator as~\citet{electra}.
All the hyperparameter configurations are the same as \ele{}.
We report the median performance on the dev set over five different random seeds. 
Refer to the appendix for more details about fine-tuning.

Results on SQuAD 2.0 are shown in Table~\ref{tbl:sq2_res}.
Consistently, our models perform better than \ele{} baseline under both the small-size setting and the base-size setting.
Under the base setting, our models improve the performance over the reimplemented \ele{} baseline by 0.6 points (EM) and 0.5 points (F1).
Especially under the small setting, our models outperform the baseline by a remarkable margin.
\ele{}+HP$_{\rm Dist}$+Focal obtains 4.3 and 4.2 points absolute improvements on EM and F1 metric.

\begin{table}[t]
\centering
\begin{tabular}{l c c }
\toprule
\multirow{2}*{\textbf{Model}} & \multicolumn{2}{c}{\textbf{SQuAD 2.0}} \\ 
& EM & F1 \\ \midrule
~~\emph{Small-size} & \\
\ele{} (\emph{reimplementation}) & 68.9 & 71.3  \\ 
\midrule
\ele{}+HP$_{\rm Loss}$+Focal & 71.8 & 74.3\\
\ele{}+HP$_{\rm Dist}$+Focal & \textbf{73.2} & \textbf{75.5}\\
\midrule
~~\emph{Base-size} & \\
BERT~\cite{bert} & 73.7 & 77.1 \\
RoBERTa~\cite{roberta} & - & 79.7 \\
XLNet~\cite{xlnet} & 78.2 & 81.0 \\
\ele{}~\cite{electra} & 80.5 & 83.3 \\
\electric{}~\cite{electric} & 80.1 & - \\
\ele{} (\emph{reimplementation}) & 82.4 & 85.1  \\
\midrule
\ele{}+HP$_{\rm Loss}$+Focal & \textbf{83.0} & \textbf{85.6} \\
\ele{}+HP$_{\rm Dist}$+Focal & 82.7 & 85.4 \\
\bottomrule
\end{tabular}
\caption{Comparisons between our models and previous pretrained models on SQuAD 2.0 dev set.
Reported results are medians over five random seeds.}
\label{tbl:sq2_res}
\end{table}

\begin{table}[t]
\centering
\begin{tabular}{l c c}
\toprule
\textbf{Model}   & \textbf{MNLI-m} & \textbf{SQuAD 2.0}   \\ 
\midrule
\ele{} & 80.3 & 71.3 \\
\midrule
\small{\emph{\ele{} + HP$_{\rm Loss}$}}    \\
$\lambda_1 = 5$    & 80.9 & 74.2 \\
$\lambda_1 = 10$   & 81.1 & 75.1 \\
$\lambda_1 = 20$   & 81.0 & 74.1 \\
\midrule
\multicolumn{2}{c}{\small{\emph{\ele{} + HP$_{\rm Loss}$ ($\lambda_1 = 5$) + Focal} }}  \\
$\gamma=
\scaleobj{0.75}
{
\begin{cases}
3& p_{\subG} > 0.2\\
5& p_{\subG} \leq 0.2
\end{cases}}$   & 81.0 & 74.7 \\
$\gamma = 1.0$   & 81.2 & 74.3 \\
$\gamma = 4.0$   & 80.9 & 74.2 \\
\bottomrule
\end{tabular}
\caption{Ablation studies on small-size models. 
We analyze the effect of the hardness prediction loss weight $\lambda_1$ and the focal loss factor $\gamma$.
Reported results are medians over five random seeds.}
\label{tbl:ablation}
\end{table}

\subsection{Ablation Studies}

We conduct ablation studies on small-size \ele{}+HP$_{\rm Loss}$+Focal models.
We investigate the effect of the loss weight $\lambda_1$ of the sampling head and the focal loss factor $\gamma$ in order to better understand their relative importance.
Results are presented in Table~\ref{tbl:ablation}.

We first disable the focal loss and only understand the effect of $\lambda_1$.
As shown in Table~\ref{tbl:ablation}, no matter what the value of $\lambda_1$ is, our models exceed the baseline by a substantial margin, which demonstrates that the hardness prediction can indeed improve the pre-training and our methods are not sensitive to the loss weight hyperparameter.
Next, we fix $\lambda_1$ at 5 and understand the effect of the focal loss factor $\gamma$.
We observe that the application of the focal loss with piecewise $\gamma = \mathbbm{1}(p_{\subG}\!>\!0.2) * 3 + \mathbbm{1}(p_{\subG} \leq 0.2) * 5$ and $\gamma=1$ can improve the performance on two datasets, which proves the effectiveness of the sampling smoothing.

\section{Analysis}
To better understand the main advantages of our models over \ele{}, we conduct several analysis experiments.

\subsection{Impacts on Sampling Distributions}

\begin{figure}[t!]
\centering
\begin{subfigure}[t]{0.46\textwidth}
\centering
\includegraphics[width=\textwidth]{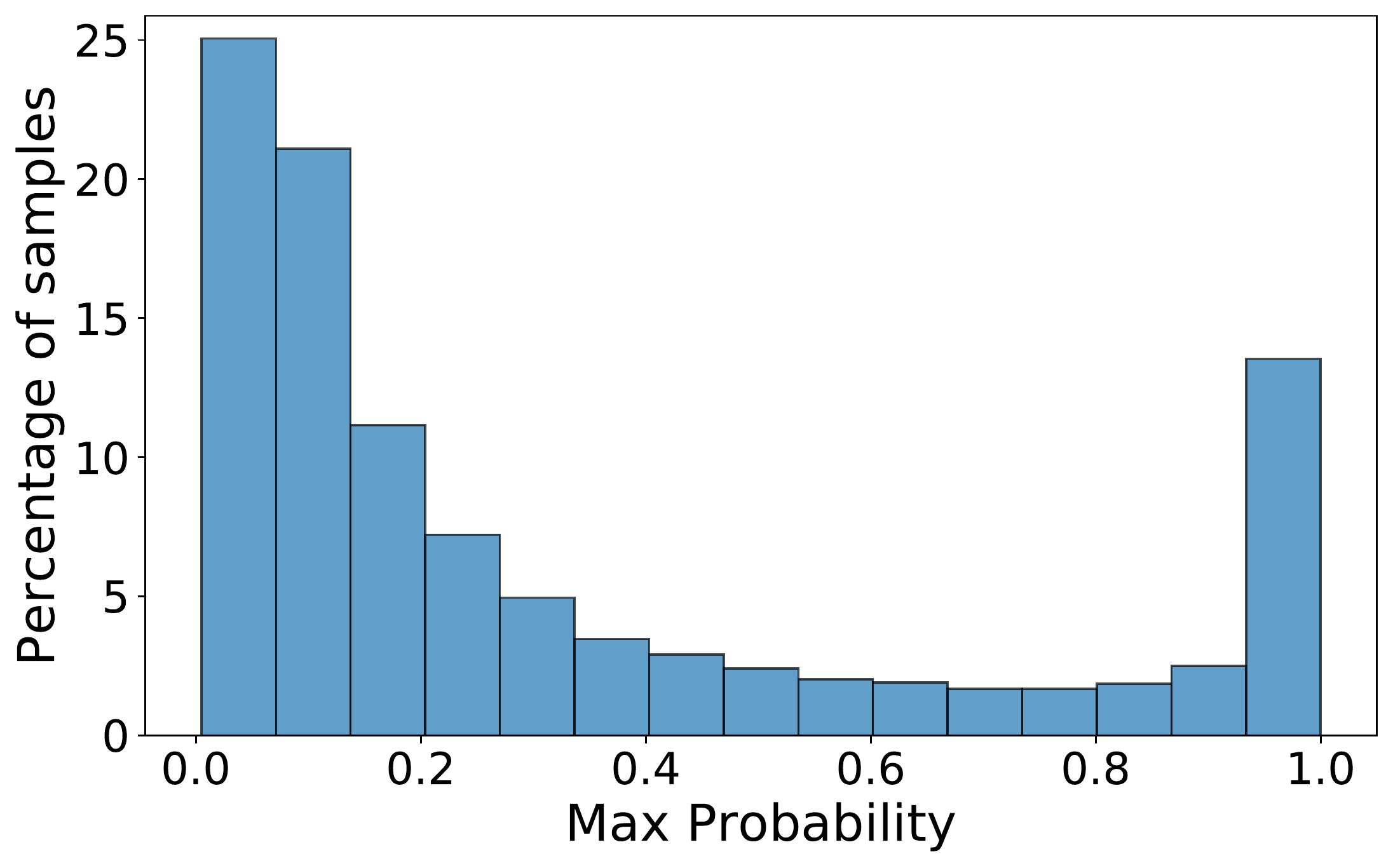}
\caption{Sampling Distribution of \ele{}}
\label{fig:conf_pg}
\end{subfigure}
\hfill
\begin{subfigure}[t]{0.46\textwidth}
\centering
\includegraphics[width=\textwidth]{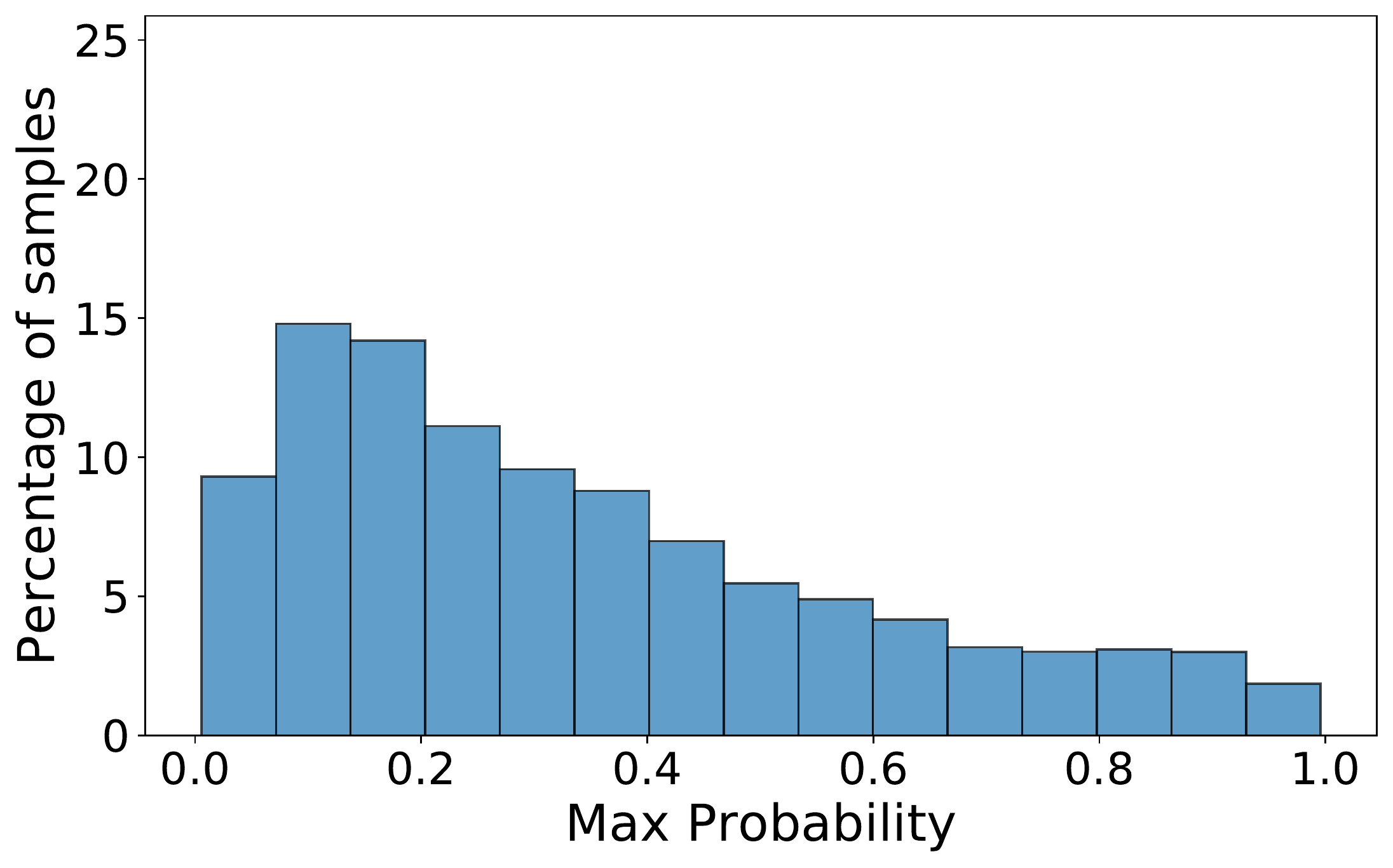}
\caption{Sampling distribution of our models}
\label{fig:conf_ps}
\end{subfigure}
\caption{The distribution of the maximum probability at the masked positions under the sampling distributions of \ele{} and our models.}
\label{fig:conf} 
\end{figure}

We first provide a comparison between the sampling distributions of \ele{} and our models illustrate the effect of our proposed methods.
We conduct evaluations on a subset of the pre-training corpus.
Figure~\ref{fig:conf} demonstrates the distribution of the maximum probability of the two sampling distributions at the masked positions.
We observe that the ratio of the maximum value under \ele{} sampling distribution between $[0.9, 1]$ is much higher than that of our models.
In other words, the original distribution suffers from over-sampling the high-probability tokens and the discriminator is forced to learn from these easy examples repeatedly.
In contrast, the distribution of the maximum value of our models in each interval is relatively more uniform than \ele{}, which indicates that our methods can significantly reduce the probability of sampling the well-classified tokens and smooth the whole sampling distribution.

\subsection{Estimation Quality}

\begin{table}[t]
\centering
\begin{tabular}{l | c c c}
\toprule
\textbf{Token Type} & {Original} & {Replaced} & {All}  \\ 
\midrule
\textbf{Corr. Coeff.} &  0.78 & 0.61 & 0.64 \\
\bottomrule
\end{tabular}
\caption{Correlation coefficient between the actual discriminator loss and the estimated value for \ele{}+HP$_{\rm Loss}$+Focal. ``Original'': sampling correct tokens as replacements. ``Replaced'': the positions that are substituted to incorrect tokens.}
\label{tbl:corr}
\end{table}

In order to measure the estimation quality of discriminator loss, we evaluate our models on a held-out set of pre-training corpus and compute the correlation coefficient between the actual discriminator loss $\mathcal{L}_D(x,c)$ and the estimated value $\hat{\mathcal{L}}_D(x,c)$.
The results of \ele{}+HP$_{\rm Loss}$+Focal are shown in Table~\ref{tbl:corr}.
We report the estimation quality of the original tokens and the replaced tokens separately.
The correlation coefficient value is $0.64$ over two types of tokens, which proves that $\mathcal{L}_D(x,c)$ and $\hat{\mathcal{L}}_D(x,c)$ correlate well.
Furthermore, we observe that the estimation quality over the original tokens is relatively higher than the replacements.
We speculate that the sampling probability of the original tokens is generally higher than the replacements, so the sampling head tends to receive more feedback from these original tokens.

\subsection{Prediction Accuracy of the Discriminator}

\begin{table}[t]
\centering
\begin{tabular}{l c c}
\toprule
\textbf{Model} & \textbf{Masked Positions} & \textbf{All Positions}  \\ 
\midrule
\ele{} &  0.81 & 0.96 \\
Ours &  0.72 & 0.95 \\
\bottomrule
\end{tabular}
\caption{Replacement detection accuracy of \ele{} and \ele{}+HP$_{\rm Loss}$+Focal. 
The models are evaluated on $15$\% masked positions and all input tokens respectively.
Our method samples more hard examples.
}
\label{tbl:pred_acc}
\end{table}

In order to verify the claim that the sampling distribution of our models indeed considers the detection difficulty, we evaluate the prediction accuracy of the discriminator under the two sampling schemes of \ele{} and \ele{}+HP$_{\rm Loss}$+Focal. 
Results are listed in Table~\ref{tbl:pred_acc}.
No matter evaluating at all positions or only at the masked positions, the detection accuracy under our sampling distribution is relatively lower than under masked language modeling in original \ele{}.
Because the unmasked tokens constitute the majority of input examples, the difference of the all-token accuracy between two models is not so distinct compared to the masked tokens.
This phenomenon is consistent with our original intention. It proves that our models can sample more replacements that the discriminator struggles to make correct predictions.
In contrast, the replacements sampled from \ele{} are easier to distinguish.

\section{Conclusion}

We propose to improve the replacement sampling for \ele{} pre-training.
We introduce two methods, namely hardness prediction and sampling smoothing.
Rather than sampling from masked language modeling, we design a new sampling scheme, which considers both the MLM probability and the prediction difficulty of the discriminator.
So the generator can receive feedback from the discriminator.
Moreover, we adopt the focal loss to MLM, which adaptively downweights the well-classified examples and smooth the entire distribution.
The sampling smoothing technique relieves oversampling original tokens as replacements.
Results show that our models outperform \ele{} baseline.
In the future, we would like to apply our strategies to other pre-training frameworks and cross-lingual models.
Moreover, we are exploring how to integrate the findings and insights of the proposed method into the masked language modeling task, which seems also quite promising.

\bibliographystyle{acl_natbib}
\bibliography{anthology,acl2021,ele}

\appendix

\section{GLUE Details}
GLUE~\cite{glue} is a collection of various natural language understanding (NLU) tasks, including inference tasks (MNLI, QNLI and RTE), similarity and paraphrase tasks (MRPC, QQP and STS-B), and single-sentence tasks (CoLA and SST-2).

\paragraph{MNLI}
The Multi-Genre Natural Language Inference Corpus~\cite{mnli2017} is a crowdsourced collection of sentence pairs with textual entailment annotations. Given a premise sentence and a hypothesis sentence, the task is to predict whether the premise entails the hypothesis (entailment), contradicts the hypothesis (contradiction), or neither (neutral). 
The dataset contains 393k train examples drawn from ten different sources.

\paragraph{QNLI}
The Stanford Question Answering Dataset~\cite{squad1} is a question-answering
dataset consisting of question-paragraph pairs.
The task is to predict whether a context sentence contains the answer to a question sentence.
The dataset contains 108k train examples from Wikipedia.

\begin{table*}[t]
\centering
% \small
\begin{tabular}{l c c}
\toprule
  \textbf{Hyperparameter}   & \textbf{Value} \\ \midrule
  Learning Rate & 3e-4 for Small, \{5e-5, 1e-4, 1.5e-4\} for Base \\
  Adam $\epsilon$ & 1e-6 \\
  Adam $\beta_1$ & 0.9  \\
  Adam $\beta_2$ & 0.999  \\
  Layerwise LR decay & 0.8  \\
  Learning rate decay & Linear \\
  Warmup fraction & 0.1 \\
  Attention Dropout & 0.1 \\
  Dropout & 0.1 \\
  Weight Decay & 0 for Small, \{0, 0.01\} for Base \\
  Batch Size & 32 \\
  Train Epochs & \{10, 15, 20\} for RTE, \{3, 4, 5\} for other tasks \\
\bottomrule
\end{tabular}
\caption{Fine-tuning details about ELECTRA baseline and our models.}
\label{tbl:finetuning}
\end{table*}

\paragraph{QQP}
The Quora Question Pairs dataset is a collection of question pairs from the community question-answering website Quora.
The task is to determine whether a pair of questions are semantically equivalent. 
The dataset contains 364k train examples.

\paragraph{RTE}
The Recognizing Textual Entailment datasets~\cite{rte1,rte2,rte3,rte5} come from a series of annual textual entailment challenges.
Given a premise sentence and a hypothesis sentence, the task is to predict whether the premise entails the hypothesis or not.
The dataset contains 2.5k train examples from a series of annual textual entailment challenges.

\paragraph{SST}
The Stanford Sentiment Treebank~\cite{sst2013} consists of sentences from movie reviews and human annotations of their sentiment. The task is to predict the sentiment of a given sentence.
The dataset contains 67k train examples.

\begin{table}[t]
\centering
% \small
\begin{tabular}{l c c}
\toprule
  \textbf{Hyperparameter}   & \textbf{Small} & \textbf{Base}   \\ \midrule
  Number of layers & 12 & 12 \\
  Hidden Size & 256 & 768 \\
  FFN inner hidden size &  1024 & 3072 \\
  Attention heads & 4 & 12 \\
  Attention head size & 64 & 64 \\
  Embedding Size & 128 & 768 \\
  Generator Size & 1/4 & 1/3 \\
  Mask percent & 15 & 15 \\
  Learning Rate Decay &  Linear &  Linear \\
  Warmup steps &  10000 &  10000 \\
  Learning Rate &  5e-4 & 1e-4 \\
  Adam $\epsilon$ & 1e-6 & 1e-6 \\
  Adam $\beta_1$ & 0.9 & 0.9 \\
  Adam $\beta_2$ & 0.999 & 0.999 \\
\bottomrule
\end{tabular}
\caption{Pre-training details about ELECTRA baseline and our models.}
\label{tbl:pretraining}
\end{table}

\paragraph{MRPC}
The Microsoft Research Paraphrase Corpus~\cite{mrpc2005} is a corpus of sentence pairs automatically extracted from online news sources.
The task is to predict whether two sentences are semantically equivalent or not.
The dataset contains 3.7k train examples.

\paragraph{CoLA}
The Corpus of Linguistic Acceptability~\cite{cola2018} consists of English acceptability judgments drawn from books and journal articles on linguistic theory.
The task is to determine whether a given sentence is grammatical or not. 
The dataset contains 8.5k train examples.

\paragraph{STS-B}
The Semantic Textual Similarity Benchmark~\cite{sts-b2017} is a collection of sentence pairs drawn from news headlines, video and image captions, and natural language inference data.
The tasks is to predict how semantically similar two sentences are on a 1-5 scale.
The dataset contains 5.8k train examples.

\section{Pre-training Details}

We did not search any hyperparameters during pre-training.
Most of our pre-training configurations are same as the original ELECTRA~\cite{electra}.
The learning rate for base-sized model is changed from 2e-5 to 1e-5 on both ELECTRA baseline and our models, because we expect the fair comparison with BERT and RoBERTa.
We keep $\lambda_2=50$ for both ELECTRA-EL and ELECTRA-AD.
The full set of pre-training hyperparameters is provided in Table~\ref{tbl:pretraining}.

\section{Fine-tuning Details}

For base-sized models, we searched the learning rate and pre-training epochs on both ELECTRA baseline and our models.
For small-sized models, we use the same hyperparameters as suggested in ELECTRA.
All the results come from the single-task fine-tuning.
The full set of fine-tuning hyperparameters is provided in Table~\ref{tbl:finetuning}.

\section{Prediction Accuracy of the Discriminator}
\begin{figure}[h]
    \centering
    \includegraphics[width=0.96\linewidth]{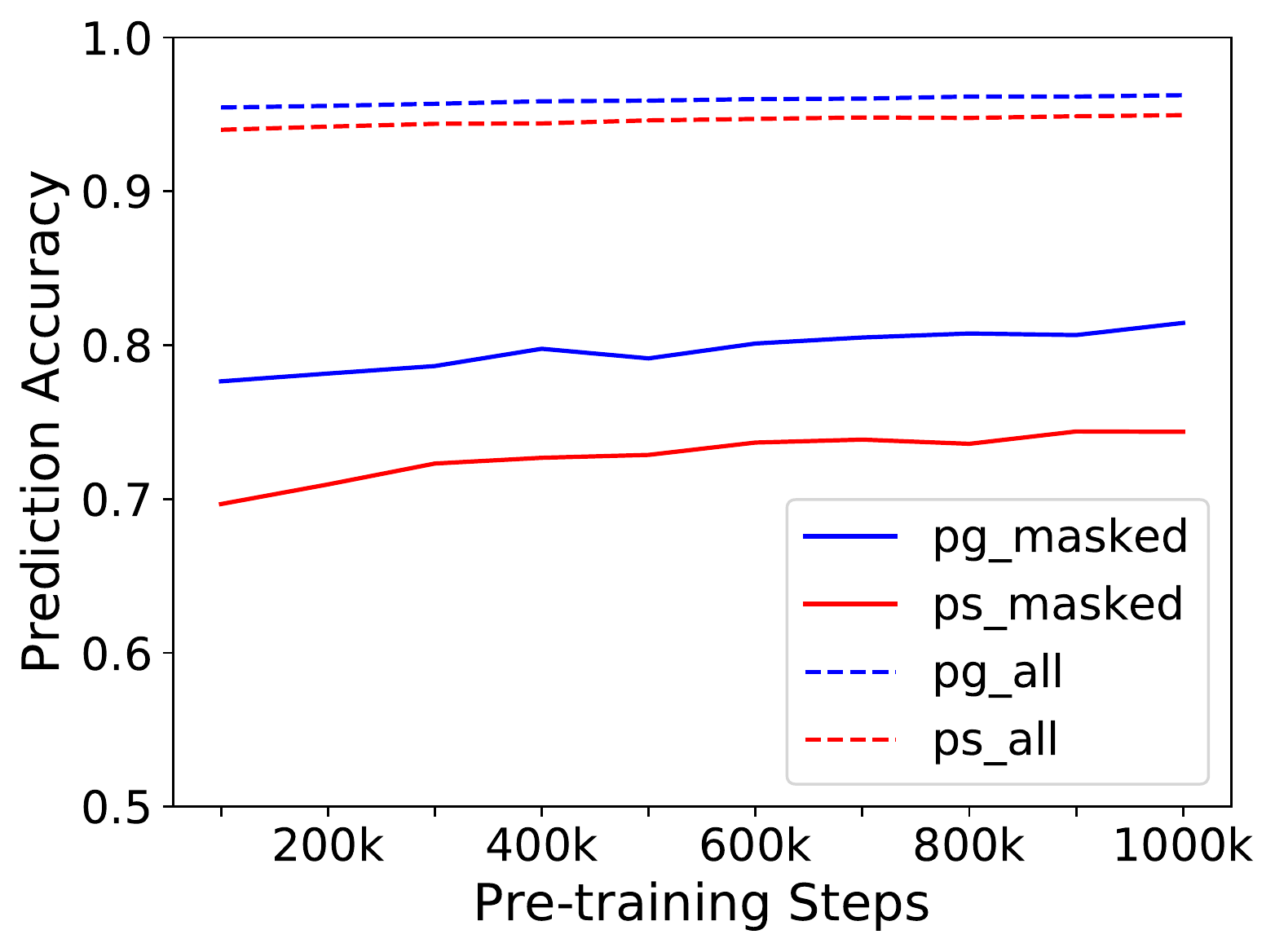}
    \caption{The prediction accuracy of the discriminator. Blue lines indicate sampling replacements according to $p_{\subG}$, red lines are according to $p_{\subS}$. The solid line represents the prediction accuracy evaluated on the 15\% masked tokens and the dashed line represents the prediction accuracy of all input tokens.}
    \label{fig:pred_acc}
\end{figure}

\end{document}

%% file: settings.tex
\usepackage{multirow}
\usepackage{amsmath}
\usepackage{capt-of}
\usepackage{tabularx}
\usepackage{epsfig}
\usepackage{amssymb}
\usepackage{amsfonts}
\usepackage{booktabs}
\usepackage{scalerel}
\usepackage[inline]{enumitem}
\usepackage{listings}
\usepackage{varwidth}
\usepackage[export]{adjustbox}
\usepackage{tikz}
\usetikzlibrary{tikzmark}
\usepackage{cleveref}

\usepackage{stmaryrd}
\usepackage{bbm}

\newcommand{\sptk}[1]{\texttt{[#1]}}
\newcommand{\eqform}[1]{Equation~(\ref{#1})}

%% file: math_commands.tex
\usepackage{amsmath,amsfonts,bm}

\def\eqref#1{equation~\ref{#1}}
\def\1{\bm{1}}

\DeclareMathAlphabet{\mathsfit}{\encodingdefault}{\sfdefault}{m}{sl}
\SetMathAlphabet{\mathsfit}{bold}{\encodingdefault}{\sfdefault}{bx}{n}

\newcommand{\softmax}{\mathrm{softmax}}

\newcommand{\Var}{\mathrm{Var}}